\documentclass{article}
\usepackage{ijcai17}

\usepackage{times}

\newif\ifappendix
\appendixtrue

\ifappendix
\usepackage{algorithm,algorithmicx,algpseudocode}
\fi

\usepackage{amsmath,amsthm,amssymb,mathtools,subfigure}

\newtheorem{proposition}{Proposition}

\newtheorem{definition}{Definition}
\newtheorem{example}{Example}

\usepackage{xcolor}
\usepackage{url}
\newcommand{\red}[1]{\textcolor{red}{#1}}
\newcommand{\todo}[1]{\red{\textsc{todo:} #1}}

\newcommand{\green}[1]{#1}

\newcommand{\alice}[0]{\textit{alice}}
\newcommand{\bob}[0]{\textit{bob}}
\newcommand{\eve}[0]{\textit{eve}}
\newcommand{\sm}[0]{\textit{sm}}
\newcommand{\fr}[0]{\textit{fr}}

\allowdisplaybreaks
\usepackage{graphicx}

\usepackage{tikz}
\usetikzlibrary{positioning}
\usetikzlibrary{decorations.text}
\usetikzlibrary{arrows}
\usetikzlibrary{shapes}

\usepackage{pgfplots}
\usepackage{multirow}
\pgfplotsset{compat=1.13}

\title{Induction of Interpretable Possibilistic Logic Theories from Relational Data}

\author{Ond\v{r}ej Ku\v{z}elka \\ 
Cardiff University, UK  \\
KuzelkaO@cardiff.ac.uk \And 
Jesse Davis \\ 
KU Leuven, Belgium  \\
jesse.davis@cs.kuleuven.be \And
Steven Schockaert \\ 
Cardiff University, UK  \\
SchockaertS1@cardiff.ac.uk}

\begin{document}

\maketitle

\begin{abstract}
The field of Statistical Relational Learning (SRL) is concerned with learning probabilistic models from relational data. Learned SRL models are typically represented using some kind of weighted logical formulas, which make them considerably more interpretable than those obtained by e.g.\ neural networks. In practice, however, these models are often still difficult to interpret correctly, as they can contain many formulas that interact in non-trivial ways and weights do not always have an intuitive meaning. To address this, we propose a new SRL method which uses possibilistic logic to encode relational models. Learned models are then essentially stratified classical theories, which explicitly encode what can be derived with a given level of certainty. Compared to Markov Logic Networks (MLNs), our method is faster and produces considerably more interpretable models. 
\end{abstract}

\section{Introduction}
The aim of Statistical Relational Learning (SRL) is to learn models that can make predictions from sets of relational facts.
Many popular SRL frameworks, such as Markov Logic Networks (MLNs \cite{Richardson2006}), \green{probabilistic soft logic \cite{psl}}, and various forms of probabilistic logic programs \cite{DBLP:journals/ml/RaedtK15}, use weighted logical formulas to encode the statistical regularities that have been observed in training data. 
Despite the use of logical formulas, learned models are often surprisingly hard to interpret. Consider, for instance, the following fragment of an MLN that was learned from the UWCSE dataset\footnote{\url{https://alchemy.cs.washington.edu/data/uw-cse/}}:
\begin{align*}
-5.11:&\,\, \textit{student}(X)\\
5.11: &\,\, \textit{professor}(X)\\
-12.01: &\,\, \neg \textit{student}(X) \vee \textit{faculty-adj}(X) \vee \textit{professor}(X)\\
&\quad\quad \vee \textit{faculty}(X) \vee \textit{faculty-aff}(X)
\end{align*}
The first two formulas intuitively mean that, all things being equal, a given individual is unlikely to be a student and likely to be a professor. However, if a given individual is a professor, then the third formula becomes satisfied. Due to the large negative weight of this formula, it turns out that being a student is actually considered to be more likely than being a professor. In practice, there can be many formulas that interact in such a way, making it hard to predict the behavior of the MLN by inspecting the weighted formulas. This limits the usefulness of MLNs for explorative data analysis, and makes it almost impossible for domain experts to manually tweak a learned MLN. 
\green{Probabilistic logic programming (PLP) languages, such as ProbLog, attach probabilities to either rules or facts.
For programs with neither {\em negation as failure} nor cyclic dependencies, the individual probabilities have a clearer meaning than in MLNs in that a rule such as $0.5$ :: $a$ :- $b$ can be interpreted as if $b$ is true, $a$ is true with at least 50\% probability. However, using negation as failure and cyclic dependencies often leads to counter-intuitive PLPs where rules cannot be understood in isolation (e.g.\ see Section 8.3 in \cite{buchman.2017}). Yet, even for propositional PLPs, excluding negation as failure limits the expressivity of the language \cite{buchman.2017}.
}
While the interpretability of AI systems is becoming increasingly important \cite{baehrens2010explain,sanchez2015towards,ribeiro2016should}, we are not aware of any existing methods for learning joint relational models that focus on interpretability.

Possibilistic logic \cite{Lang:1991:LGP:2100662.2100687} also uses weighted formulas, usually written as $(\alpha,\lambda)$ with $\alpha$ a classical formula and $\lambda \in [0,1]$ a certainty weight. As suggested in \cite{kuzelka.ecai.2016}, we can use \emph{possibilistic} logic to encode \emph{probability} distributions. The formula $(\alpha,\lambda)$ then expresses the constraint that the probability of any world violating $\alpha$ can be at most $1-\lambda$. Because of this constraint based semantics, formulas can safely be interpreted in isolation from the rest of the theory, which we believe is crucial for interpretability. The method proposed in \cite{kuzelka.ecai.2016} derives a possibilistic logic theory from a density estimation tree \cite{ram2011density}, which is in turn learned from a set of training examples. Compared to Markov Random Fields (MRFs), the possibilistic logic theories resulted in a higher accuracy for Maximum A Posteriori (MAP) queries with small evidence sets, while MRFs were more accurate for larger evidence sets. Essentially, inference from possibilistic logic theories captures the conclusions that we can obtain by applying a form of commonsense reasoning (see \cite{DBLP:conf/uai/KuzelkaDS15} for a theoretical justification of this view). In the presence of large amounts of evidence, however, MRFs can make predictions even when there is no obvious ``default knowledge'' that applies, by aggregating large amounts of individually weak and/or conflicting pieces of evidence.

In this paper we introduce a method for learning possibilistic logic theories from relational data. In principle, such theories could be learned by ``lifting'' the approach from \cite{kuzelka.ecai.2016} to the relational setting. However, this technique relies on identifying a set of formulas $\alpha_1,...,\alpha_n$ which are mutually exclusive and jointly exhaustive (corresponding to the branches of the density estimation tree). In a relational setting, this essentially requires us to enumerate isomorphism classes, of which there are typically exponentially many. As a result, the  possibilistic logic theories we obtain quickly become prohibitively large (even though these theories could subsequently be pruned). Therefore, we follow a different strategy in this paper. To obtain suitable formulas, we first learn a set of hard constraints. These hard constraints allow us to generate non-trivial negative examples, which together with the positive examples obtained from the training data, allow us to learn a set of Horn rules that describe how the different predicates relate to each other. The restriction to Horn rules increases the interpretability of the learned theories, and leads to theories that are optimized for predicting positive literals, which is usually what is needed in applications. Note that the hard constraints are not restricted to Horn rules, which means that our theories can still be used to predict negative literals. In the last step, we use a form of relational model counting to associate a weight with each of the learned Horn rules.

To the best of our knowledge, our approach is the first that represents joint relational models in such a way that each weighted formula can be interpreted in isolation. 
The most closely related work is \cite{Serrurier2007939}, where first-order possibilistic logic theories were learned in an Inductive Logic Programming (ILP) setting. However, the learned theories from that work are aimed at predicting a single target predicate. Moreover, their approach was based on a non-standard semantics for possibilistic logic, in which formulas cannot be interpreted in isolation. Finally, their approach is purely qualitative, i.e.\ formulas are ranked but are not given weights with a probabilistic interpretation.

\section{Preliminaries}
Throughout the paper, we consider a function-free first-order logic language $\mathcal{L}$, which is built from a set of constants $\textit{Const}$, variables $\textit{Var}$ and predicates $\textit{Rel} = \bigcup_i \textit{Rel}_i$, where $\textit{Rel}_i$ contains the predicates of arity $i$. For $a_1,...,a_k \in \textit{Const}\cup \textit{Var}$ and $R \in \textit{Rel}_k$, we call $R(a_1,...,a_k)$ an atom.  If $a_1,..,a_k\in \textit{Const}$, this atom is called ground. A literal is an atom or the negation of an atom, and a clause is a disjunction of literals. The formula $\alpha_0$ is called a grounding of $\alpha$ if $\alpha_0$ can be obtained from $\alpha$ by substituting each variable by a particular constant from $\textit{Const}$. A formula is called closed if all variables are bound by a quantifier. A possible world $\omega$ is defined as a set of ground atoms. The satisfaction relation $\models$ is defined in the usual way.

A Markov logic network (MLN)~\cite{Richardson2006} is a set of weighted formulas $w: F$, with $w\in \mathbb{R}$ and $F$ a function-free and quantifier-free first-order formula. The semantics are defined w.r.t.\ the groundings of the first-order formulas, relative to some finite set of constants. An MLN is seen as a template that defines an MRF. Specifically, an MLN $\mathcal{M}$ induces the following probability distribution on the set of possible worlds $\omega$: 
\begin{eqnarray}
p_{\mathcal{M}}(\omega) = \frac{1}{Z}\mbox{exp}\left(\sum_{w:F \in \mathcal{M}} w n_F(\omega)\right)
\label{e:mln}
\end{eqnarray}
\noindent
where $n_F(x)$ is the number of groundings of $F$ that are satisfied in $\omega$, and $Z$ is a normalization constant to ensure that $p_{\mathcal{M}}$ is a  probability distribution. A key inference task for MLNs is computing the Maximum A Posteriori (MAP) consequences, i.e.\ determining which ground atoms are true in the most probable models of a given set of ground atoms. Formally, $(\mathcal{M},E)\models \alpha$, for $E$ a set of ground atoms and $\alpha$ a ground atom, iff $\forall \omega \,.\, (p_{\mathcal{M}}(\omega) = \max_{\omega'} p_{\mathcal{M}}(\omega')) \Rightarrow (\omega\models \alpha)$.

A possibilistic logic theory \cite{Lang:1991:LGP:2100662.2100687} is a set of weighted formulas $(\alpha,\lambda)$ with $\alpha$ a propositional formula and $\lambda \in [0,1]$. A possibilistic logic theory $\Theta$ induces a mapping $\pi:\Omega \rightarrow [0,1]$, with $\Omega$ the set of  propositional interpretations, which is defined for $\omega \in \Omega$ as:
\begin{align}\label{defPossDistrTheory}
\pi_{\Theta}(\omega) = \min\{ 1- \lambda \,|\, (\alpha,\lambda)\in\Omega, \omega\not\models\alpha\}
\end{align}
The distribution $\pi_{\Theta}$ is called a possibility distribution. The possibilistic logic theories we consider will be constructed such that $\sum_{\omega }\pi_{\Theta}(\omega)=1$, in which case $\pi_{\Theta}$ can be interpreted as a probability distribution.
There is a common inconsistency-tolerant inference relation in possibilistic logic, which is actually the direct counterpart of MAP inference. Specifically, for $E$ a set of propositional formulas and $\alpha$ a propositional formula, we write $(\Theta,E) \models \alpha$ if $\forall \omega \,.\, (\pi_{\Theta}(\omega) = \max_{\omega'} \pi_{\Theta}(\omega')) \Rightarrow (\omega\models \alpha)$. Interestingly, the formulas $\alpha$ which are entailed in this sense can easily be determined syntactically. In particular, for $\mu\in[0,1]$ let $\Theta_{\mu} = \{\alpha \,|\, (\alpha,\lambda)\in\Theta, \lambda\geq \mu\}$. Let $\mu_0$ be the smallest threshold for which $\Theta_{\mu_0} \cup E$ is consistent. Then $(\Theta,E) \models \alpha$ iff $\Theta_{\mu_0} \cup E \models \alpha$. Hence inference in possibilistic logic can straightforwardly be implemented using a SAT solver.

 In this paper we will learn possibilistic logic theories with first-order formulas instead of propositional formulas. Like MLNs, these first-order possibilistic logic theories should simply be seen as templates for normal (propositional) possibilistic logic theories that are obtained by replacing each weighted first-order formula $(\alpha,\lambda)$ by the formulas $(\alpha_1,\lambda),...,(\alpha_k,\lambda)$, with $\alpha_1,...,\alpha_k$ the groundings of $\alpha$. It is easy to see that when $p_{\mathcal{M}} = \pi_{\Theta}$ for an MLN $\mathcal{M}$ and first-order possibilistic logic theory $\Theta$, it holds that $(\mathcal{M},E)\models \alpha$ iff $(\Theta,E)\models \alpha$.  \cite{DBLP:conf/uai/KuzelkaDS15} demonstrated how to construct a possibilistic logic theory $\Theta$ from a given MLN $\mathcal{M}$, such that $p_{\mathcal{M}} = \pi_{\Theta}$. However, the resulting possibilistic logic theory is exponential in size. In practice, the possibilistic logic theories we learn from data can thus only approximate what could be encoded in an MLN. This makes MLNs potentially better equipped to make predictions from large amounts of evidence, while making possibilistic logic less prone to making spurious predictions in situations where the amount of evidence is more limited.

\section{Relational Marginals}

In the context of SRL, we are typically given a large set of ground atoms $\mathcal{A}$ as training data. This set essentially corresponds to a single example of a relational structure. Intuitively, we want to learn a probability distribution over such relational structures, but we clearly cannot estimate such a distribution from one example. The solution we propose is to construct a large number of training examples by sampling small fragments of this global relational structure, and then estimating a probability distribution over these fragments. We will refer to $\Upsilon = (\mathcal{A},\mathcal{C})$, with $\mathcal{C}$ the set of constants appearing in $\mathcal{A}$, as an example. We now explain how we can obtain a collection of ``local'' training examples, which will correspond to (isomorphism classes of) fragments of this ``global'' example.

\begin{definition}
A (global) example is a pair $(\mathcal{A},\mathcal{C})$, with $\mathcal{C}$ a set of constants and $\mathcal{A}$ a set of ground atoms which only use constants from $\mathcal{C}$. Let $\Upsilon = (\mathcal{A},\mathcal{C})$ be an example and $\mathcal{S}\subseteq \mathcal{C}$. The fragment $\Upsilon\langle S \rangle = (\mathcal{B},\mathcal{S})$ is defined as the restriction of $\Upsilon$ to the constants in $\mathcal{S}$, i.e.\ $\mathcal{B}$ is the set of all atoms from $\mathcal{A}$ which only contain constants from $\mathcal{S}$. 
\end{definition}

\noindent Intuitively, we can repeatedly sample subsets $\mathcal{S}$ and then consider each $\Upsilon \langle S \rangle$ as a training example. However, the constants appearing in each of these fragments will be different, hence to enable generalization we need to consider their isomorphism classes.

\begin{definition}[Isomorphism]
Two examples $\Upsilon_1 = (\mathcal{A}_1,\mathcal{C}_1)$ and $\Upsilon_2 = (\mathcal{A}_2,\mathcal{C}_2)$ are isomorphic, denoted as $\Upsilon_1 {\approx} \Upsilon_2$, if there exists a bijection $\sigma : \mathcal{C}_1 \rightarrow \mathcal{C}_2$ such that $\sigma(\mathcal{A}_1) = \mathcal{A}_2$, where $\sigma$ is extended to ground atoms in the usual way. 
\end{definition}

\begin{definition}[Local example]
Let $k\in \mathbb{N}$ and let $\mathcal{L}_k$ be the language which contains the same predicates and variables as $\mathcal{L}$ but only constants from the set $\{1,2,...,k\}$. A local example of width $k$ is a pair $\omega=(\mathcal{A},\{1,...,k\})$, where $\mathcal{A}$ is a set of ground atoms from the language $\mathcal{L}_k$. For an example $\Upsilon = (\mathcal{A},\mathcal{C})$ and $\mathcal{S}\subseteq \mathcal{C}$, we write $\Upsilon[\mathcal{S}]$ for the set of all local examples of width $|\mathcal{S}|$ which are isomorphic to $\Upsilon\langle \mathcal{S} \rangle$. 
\end{definition}

\noindent To distinguish local examples from global examples, we will denote them using lower case Greek letters such as $\omega$ instead of upper case letters such as $\Upsilon$.

\begin{example}\label{ex1}
For $\Upsilon = (\{ \fr(\alice,\bob),\allowbreak\fr(\bob,\alice),\allowbreak\fr(\bob,\allowbreak\eve),\allowbreak \fr(\eve,\bob),\allowbreak \sm(\alice) \},\allowbreak \{\alice,\allowbreak \bob,\allowbreak \eve \})$ we have:
\begin{align*}
\Upsilon\langle \{ \alice, \bob \} \rangle {=}& (\{ \fr(\alice,\bob), \fr(\bob,\alice),\sm(\alice) \},\\
&\quad\{\alice,\bob\})\\
\Upsilon[\{ \alice, \bob \}] {=}& \{ (\{ \fr(1,2), \fr(2,1),\sm(1) \}, \{1,2\}) \\
 &\phantom{\{} (\{ \fr(2,1), \fr(1,2),\sm(2) \}, \{1,2\}) \}
\end{align*}

\end{example}

\noindent We can now naturally define a probability distribution over local examples of width $k$.
\begin{definition}[Relational marginal distribution]
Let $\Upsilon = (\mathcal{A},\mathcal{C})$ be an example and $k\in \mathbb{N}$. The relational marginal distribution of $\Upsilon$ of width $k$ is a distribution $P_{\Upsilon,k}$ over local examples, where $P_{\Upsilon,k}(\omega)$ is defined as the probability that $\omega$ is sampled by the following process:
\begin{enumerate}
\item Uniformly sample a subset $\mathcal{S}$ of $k$ constants from $\mathcal{C}$.
\item Uniformly sample a local example $\omega$ from the set $\Upsilon[\mathcal{S}]$.
\end{enumerate}
For a closed formula $\alpha$, we also define:
$$
P_{\Upsilon,k}(\alpha) = \sum_{\omega : \omega \models \alpha} P_{\Upsilon,k}(\omega)
$$
\end{definition}



\noindent 



\noindent In the following, constant-free existentially-quantified conjunctions of atoms will play an important role, as they are the syntactic counterpart of the isomorphism classes $\Upsilon[\mathcal{S}]$. For such a conjunction $\alpha$, it holds that $P_{\Upsilon,k}(\alpha)$ is equal to the probability that a randomly sampled set $\mathcal{S}$ of $k$ constants satisfies $\Upsilon \langle \mathcal{S} \rangle \models \alpha$. In this sense, relational marginal distributions faithfully model the probabilities of isomorphism classes of local examples. Naturally, other probability distributions on local examples might also faithfully model the probabilities of these isomorphism classes, but it is easy to see that relational marginal distributions have the highest entropy among such models. 

The idea of relational marginals is similar to the random selection semantics used in \cite{Schulte2014}, but the difference is that for relational marginals, we restrict the sample sets to have fixed cardinality and then standardize them as local examples. This allows us to construct a standard probability distribution over local examples.

\section{Possibilistic Logic Encoding of Relational Marginals}

In this section we describe how relational marginals can be encoded in possibilistic logic. As we show first, in principle we can use a direct generalization of the approach from \cite{kuzelka.ecai.2016}, by taking advantage of the fact that each isomorphism class $\Upsilon[\mathcal{S}]$ of local examples corresponds to a constant-free existentially-quantified conjunction of atoms. For an example $\Upsilon$, let $g_k(\Upsilon)=\{\alpha_1,...,\alpha_n\}$ be a set that contains one such formula for each isomorphism class of local examples of width $k$.

\begin{definition}[Possibilistic encoding of relational marginals]\label{def1}
Let $\Upsilon$ be an example and let $k \in \mathbb{N}$. The possibilistic logic theory corresponding to $P_{\Upsilon,k}$ is defined as
$$
\Theta_{\Upsilon,k} = \left\{ \left. \left(\neg \alpha, 1-\frac{1}{c(\alpha)} P_{\Upsilon,k}(\alpha) \right) \right| \alpha \in g_k(\Upsilon) \right\}
$$
where $c(\alpha)$ is the cardinality of the isomorphism class represented by $\alpha$.
\end{definition}


\begin{proposition}
Let $\Upsilon$ be an example, $k \in \mathbb{N}$, and $\omega$ a local example of width $k$. It holds that $P_{\Upsilon,k}(\omega) = \pi(\omega)$ where $\pi(.)$ is the possibility distribution associated with $\Theta_{\Upsilon,k}$.
\end{proposition}
\begin{proof}
Let $\omega$ be a local example of width $k$. By definition, $g_k(\Upsilon)$ contains a unique formula $\alpha^*$ such that $\omega \models \alpha^*$, since the formulas in $g_k(\Upsilon)$ define a partition of local examples into isomorphism classes.
Accordingly, $\neg \alpha^*$ is the unique formula appearing in $\Theta_{\Upsilon,k}$ which is not satisfied by $\omega$. By \eqref{defPossDistrTheory}, we therefore have $\pi(\omega) = 1-(1-P_{\Upsilon,k}(\alpha^*)/c(\alpha^*) = P_{\Upsilon,k}(\alpha^*)/c(\alpha^*) = P_{\Upsilon,k}(\omega)$, where the last equality hold because all local examples from the same partition class have the same probability in a relational marginal distribution.
\end{proof}

\noindent The number of isomorphism classes typically grows very quickly with increasing $k$, so the exact transformation from Definition \ref{def1} can only be used for very simple problem domains. In practice, representing the relational marginal distribution exactly is typically not feasible. An exact representation would moreover not necessarily generalize well to previously unseen data. Therefore, for the remainder of this paper, we will focus on learning approximate possibilistic logic representations of relational marginal distributions. 

Specifically, our aim is to construct a possibilistic logic theory $\Theta = \{(\alpha_1,\lambda_1),...,(\alpha_n,\lambda_n)\}$ such that for the associated possibility distribution $\pi$ it holds that $\pi(\omega)$ is approximately equal to $P_{\Upsilon,k}(\omega)$. This problem can be decomposed in two steps. The first step is structure learning, i.e.\ choosing suitable formulas $\alpha_1,...,\alpha_n$. In this paper, we will only consider constant-free and quantifier-free formulas. However, recall that first-order possibilistic logic theories are seen as templates for propositional theories, which means that all variables in the formulas $\alpha_1,...,\alpha_n$ are implicitly universally quantified. The second step is weight learning. In this step, we aim to find the weights $\lambda_1,...,\lambda_n$ for which $\pi$, seen as a probability distribution, maximizes the likelihood of a set of training examples. Note that if $\lambda_1 \leq ... \leq \lambda_n$ we can assume w.l.o.g.\ that $\alpha_1 = \bot$. We need to include such a formula $\alpha_1$ to encode the probability of the most probable worlds (which is then given by $1-\lambda_1$).

As the transformation from Definition \ref{def1} illustrates, weight learning becomes very simple when using mutually exclusive formulas. However, using mutually exclusive formulas is not desirable, as such formulas quickly become very large\footnote{One exception is when $k=1$, which corresponds to the propositional case, where density estimation trees can be used, as was proposed in \cite{kuzelka.ecai.2016}.}, which also makes the resulting theories difficult to interpret. Therefore, in practice, we will rely on greedy methods for weight learning. These will be discussed in Section \ref{secWeightLearning}. 

\section{Structure Learning}


In this section, we propose a method to learn Horn rules that can be used to predict all predicates from $\Upsilon$. Using Horn rules makes the resulting possibilistic logic theories more interpretable, and allows us to optimize them for predicting atoms, which is what is usually required. Learning Horn rules using methods based on inductive logic programming \cite{ilp} typically requires both positive and negative training examples. In Subsection \ref{secConstructingExamples} we explain how to construct examples and then  discuss our method for learning Horn rules in Subsection \ref{secHornRules}.

\subsection{Constructing Training Examples}\label{secConstructingExamples}

\green{Constructing positive examples for a given predicate $P$ is straightforward: we can simply take all, or a subsample, of the true $P$-atoms from $\Upsilon$, Typically, there are significantly more negative examples than positive ones; e.g. in a typical social network there are many more examples of non-friends than of friends. Simply subsampling the negative examples is unlikely to be effective, as most of the resulting negative examples might be uninteresting, in the sense that they can be explained by some simple hard rules that hold for the domain. Hence, we first learn a set of such hard rules, and then only consider negative examples that are consistent with them.}

We are interested in hard rules that are universally quantified, constant-free clauses with no counterexamples in $\Upsilon$. We find such clauses 
by exhaustively constructing all clauses (modulo isomorphism) containing at most $t$ literals and at most $k$ variables, where $k$ is the width of the relational marginal distribution and $t$ is a parameter of the method. For each clause, we check  whether $\Upsilon \not\models \neg \alpha$ holds with a CSP solver. We store each such clause in a list if the list does not contain another clause that subsumes it. 
Because learning hard rules that only contain unary literals is typically easier than learning more general rules, we use a higher size limit $t' > t$ for these rules.

Let $\Delta$ be the set of discovered hard rules, and $\Upsilon = (\mathcal{A},\mathcal{C})$ be the global example. To select negative training examples, we reject all samples $a$ for which $\bigwedge \mathcal{A} \wedge a \wedge \Delta$ does not have a model when grounded\footnote{We use incremental grounding for efficiency.} over the set of constants $\mathcal{C}$. The result is a subsample of non-trivial negative examples. In addition, this process allows us to estimate the total number of non-trivial negative examples, which we use to compute the weight of the negative examples when estimating the accuracies of the Horn rules. 

\subsection{Learning Horn Rules}\label{secHornRules}

To find Horn rules, we employ a beam search method, which relies on two parameters: the size of the beam $b$ and the maximum number of literals in the body of a rule $l$. As before, $k$ is the width of the local examples. For a given target predicate $P$ of arity $m$, we initialize the list of candidate rules with the rule $P(X_1,...,X_m) \leftarrow \top$. 
In each iteration of the search, we construct all possible single-literal extensions of each rule in the beam such that the constraints on the number literals and variables are not violated. 
From these candidate rules, we select a set of non-isomorphic rules and evaluate their accuracy on the (weighted) sets of positive and negative examples. The algorithm then selects the $b$ most accurate rules to serve as the candidate rules for the next iteration. 
The algorithm terminates when no new candidate rules can be generated without violating the constraints on the number of literals and variables and returns the best found rule. 
This beam search method is repeated several times for each predicate $P$. 
Most rules found during one run of the beam search typically entail similar sets of examples. 
To promote diversity within each run of beam search, we discard rules that are subsumed by previously found rules.


We employ several well-known techniques to speed up the search.  
First, instead of checking isomorphism for every pair of candidate rules, we efficiently select non-isomorphic rules by hashing each one using a straightforward generalization of the Weisfeiler-Lehman labeling procedure \cite{weisfeiler1968reduction}. Then, we only check if two rules are isomorphic if they have the same hash value, and if so one of them is removed. Second, the algorithm maintains a set $\textit{Forbidden}$ of minimal rules which entailed zero positive examples in the previous iterations of the beam search. Before evaluating new candidate rules, the algorithm discards candidate rules which are subsumed by a rule from the set $\textit{Forbidden}$. Third, to reduce the negative plateau effect, known from relational learning \cite{Alphonse2008}, we add to every constructed rule a literal $\textit{AllDiff}(V_1,\dots,V_k)$, which is true iff all variables in its argument are mapped to different terms. This also improves the interpretability of the rules.


\section{Weight Learning}\label{secWeightLearning}

Let us first assume that an ordering of the formulas $(\alpha_1 = \bot, \alpha_2 \dots, \alpha_n)$ is given, and we want to learn weights $\lambda_1 \leq ...\leq \lambda_n$ which maximize the likelihood of a set of local examples $\mathcal{E}$ that have been sampled from $P_{\Upsilon,k}$. These weights can be found by solving the following optimization problem:
\begin{itemize}
\item Variables: $\lambda_1', \lambda_2', \dots, \lambda_{n}'$.
\item Maximize: $\prod_{\omega \in E} P(\omega) = \prod_{i = 1}^n (1-\lambda_i')^{|\mathcal{E}_{i+1}|-|\mathcal{E}_{i}|}$ where 
$\mathcal{E}_i = \{ \omega \in \mathcal{E} | \omega \models  \alpha_i \wedge \dots \wedge \alpha_n \}$.
\item Subject to: 
\begin{align}
\label{constr1}\lambda_1' \leq \lambda_2' \leq \dots \leq \lambda_{n}' \\
\label{constr2}\sum_{i=1 }^{k} (1-\lambda_i') \cdot \left(|M_{i+1}| - |M_{i}|\right)  = 1
\end{align} 
where $M_i = \{ \omega | \omega \models \alpha_i \wedge \dots \wedge \alpha_n \}$ \green{and (\ref{constr2}) forces probabilities of all possible worlds to sum to~$1$.}
\end{itemize}
This optimization problem can be converted to a geometric programming problem, similar to the geometric programming encoding proposed in \cite{kuzelka.ecai.2016}. Note that geometric programming problems can be converted to convex programming problems by a change of variables, and can thus be solved using standard convex programming methods \cite{Boyd2007}. 

We can think of $\mathcal{E}$ as an IID sample from the set of local examples in the multi-set $\{ \omega | \omega \in \Upsilon[\mathcal{S}], \mathcal{S} \subseteq \mathcal{C}, |\mathcal{S}| = k \}$, where $\Upsilon = (\mathcal{A},\mathcal{C})$ is the given global example. However, assuming all $\alpha_i$ are constant-free, it is easy to check that we will get the same values (in expectation) of the parameters $|\mathcal{E}_i|$ if we instead use the set $\{ \omega | \omega = \Upsilon\langle\mathcal{S}\rangle, \mathcal{S} \subseteq \mathcal{C}, |\mathcal{S}| = k \}$. 

Computing the parameters $|\mathcal{E}_{i}|$ and $|M_{i}|$ needed for weight learning is difficult (\#P-hard), so the algorithm uses a greedy approach to search for the best ordering of the formulas. It starts with a possibilistic logic theory containing only the learned hard rules. It iteratively tries to add, at each possible positions, one rule $\alpha$ from the set of candidate rules found by the structure learning algorithm. If adding the rule $\alpha$ increases the likelihood score, we keep $\alpha$ in the theory, at the position that yielded the best improvement. 
This approach permits caching and reusing many of the parameter computations (i.e. the computed parameters for many cuts of the stratified theory will be the same for many iterations of the algorithm).
During the learning process, the algorithm simplifies the constructed theories (using a relational SAT solver). It removes rules which are implied by other rules in the theory that have higher weights, and it also removes redundant literals from the individual rules.

\begin{table*}[t]
\caption{The possibilistic logic theory learned in the Yeast-Proteins dataset, not showing the hard rules and actual weights (but note that $\lambda_\bot < \lambda_1 < ... < \lambda_5$). All rules are implicitly constrained by $\textit{AllDiff}$ constraints. }\label{tab:yeast}
\begin{tabular}{ r } \hline
 \dots \quad\quad(112 hard constraints not shown here)\\
 $(\textit{Complex}(A, B) \leftarrow \textit{ProteinClass}(A, C), \textit{Interaction}(D, A), \textit{Complex}(D, B), \wedge \textit{ProteinClass}(D, C), \lambda_{5})$  \\
 \quad\quad\quad\quad$(\textit{Phenotype}(A, B) \leftarrow \textit{Interaction}(C, A) \wedge \textit{ProteinClass}(A, D) \wedge \textit{Phenotype}(C, B) \wedge \textit{ProteinClass}(C, D), \lambda_4) $ \\
 $(\textit{ProteinClass}(A, B) \leftarrow \textit{ProteinClass}(D, B) \wedge \textit{Complex}(A, C) \wedge \textit{Complex}(D, C), \lambda_3)$ \\
 $(\textit{Enzyme}(A, B) \leftarrow \textit{ProteinClass}(A, C) \wedge \textit{Interaction}(A, D) \wedge \textit{Enzyme}(D, B), \textit{ProteinClass}(D, C), \lambda_2)$ \\
 $(\textit{Location}(A, B) \leftarrow \textit{Location}(D, B) \wedge \textit{Complex}(A, C) \wedge \textit{Complex}(D, C), \lambda_1)$\\
 $(\bot, \lambda_\bot)$ \\
 \hline
\end{tabular}
\end{table*}


\section{Experiments}
We compare our approach's learned models to learned MLNs for various MAP inference tasks.  We learned MLNs using the default structure learner in the Alchemy package \cite{mln-structure-learning}.\footnote{{\tt http://alchemy.cs.washington.edu/}} For the MLNs, we used  RockIt \cite{noessner2013rockit} to perform MAP inference. 

\subsection{Methodology}

Our learning algorithm is implemented in Java and uses the SAT4j library \cite{sat4j}. Cryptominisat \cite{soos2010cryptominisat} is used for our implementation of relational version of the model counter \cite{chakraborty2016algorithmic}. It uses the JOptimizer package to solve the geometric programming problems\footnote{{\tt http://www.joptimizer.com}} needed for the maximum likelihood estimation. 

We use two standard SRL datasets: UWCSE and Yeast-Proteins.
The UWCSE dataset described relations among students, professors, papers, subjects, terms and projects in the CS department of the University of Washington. This dataset contains among other the following relations (predicates) {\em AdvisedBy/2}, {\em TempAdvisedBy/2}, {\em Publication/2}, {\em TaughtBy/3}, {\em TA/3}, {\em Student/1}, {\em Professor/1}, {\em PostQuals/1}. This dataset is split into five groups: AI, language, theory, graphics, and systems. We use AI, language and theory as a training set and graphics and systems as a test set. The Yeast-Proteins dataset contains proteins and the relations among them. We use a version in which the interaction relation is symmetric. This dataset contains the following relations: {\em Interacts/2}, {\em Enzyme/2}, {\em Complex/2}, {\em ProteinClass/2}, {\em Function/2}, {\em Phenotype} and {\em Location/2}. We randomly divide the constants (entities) in this dataset into two disjoint sets of equal size. The training set consists of atoms containing only the constants from the first set and the test set contains only the constants from the second set. This ensures that no information leaks from the training set into the test set.
 
We evaluate the performance of the learned models as follows. For each $k = 1,\dots,\textit{k}_{max}$, we sample a set of evidence literals from the test set. We then predict the MAP state by each of the learned models and compute the Hamming error, which measures the size of the symmetric difference of the predicted MAP world and the set of the literals in the test set. We then report the cumulative differences between the errors of the models, as this clearly highlights the overall trends.

\subsection{Results}

\begin{center}
\input{plots}
\end{center}
The possibilistic logic theory learned for the Yeast-Proteins dataset is shown in Table \ref{tab:yeast}. The rules seem to encode meaningful relations that hold in the dataset. For instance, if a protein $A$ is contained in a complex $C$, another protein $D$ is in $C$ as well and $D$ is at location $B$ then $A$ is also at that location. The learned MLN, on the other hand, only contained rules that model the prior probabilities of the individual predicates and one additional rule that expresses the symmetry of the $\textit{Interaction}$ relation. Hence, the only type of prediction made by this MLN consists in computing the symmetric closure of the interaction literals, which is why we do not show any separate baseline prediction for this dataset. 
The possibilistic logic theory has lower Hamming errors for evidence sets up to around 1500 literals (see Figure \ref{fig:hamming}, left panel), which can be seen from the fact that the cumulative difference is increasing over this range. For larger evidence sets, the possibilistic theory intuitively predicts ``too much'', resulting in a higher Hamming error than the MLN predictions.

The theory which was learned for the UWCSE dataset is larger, and is therefore shown in the appendix, where we also show the corresponding learned MLN. The possibilistic logic theory again contains rules which are intuitive, capturing meaningful relations for this domain. The formulas in the MLN are much harder to interpret. As shown in Figure~\ref{fig:hamming}, the possibilistic logic theory again reaches smaller Hamming errors than the learned MLN for small evidence sets, in this case for evidence sets of up to about 250 literals (right panel). It is always better than the baseline which predicts everything not in the evidence as false (middle panel).

Inference in the possibilistic logic theories is, on average, substantially faster than MAP inference in the MLNs (using RockIt). For UWCSE, the speed-up was between one and two orders of magnitude. The possibilistic logic prediction only requires us to solve a logarithmic number of SAT queries, whereas computing MLN MAP predictions requires solving a weighted MAX-SAT problem.


\section{Conclusions}
We have proposed a method for learning relational possibilistic logic theories. These theories are seen as templates for constructing ``ground'' (i.e.\ standard propositional) possibilistic logic theories, similar to how Markov logic networks can be seen as templates for constructing Markov random fields. In particular, as in standard possibilistic logic, each weighted formula has an intuitive interpretation as a constraint on the probability distribution that is being modelled. To formally describe what this probability distribution represents, we have introduced the notion of a relational marginal distribution, which we can intuitively think of as a probability distribution over fixed-sized fragments of a given relational structure. 

Our method jointly models the different predicates from the considered domain. This contrasts with inductive logic programming, which attempts to predict a single target predicate. Our approach to structure learning essentially boils down to applying ILP rule learning methods to each of the predicates in the domain. Another difference with the standard ILP setting is that we are not explicitly given a set of training examples in our setting. To generate such training examples, we explicitly learn a number of hard rules, and only consider negative examples that are non-trivial, in the sense that they cannot be explained by the hard rules alone. Interactions between the different rules are taken into account during a subsequent weight learning step.

The main design consideration of our method was to learn interpretable theories. However, as our experimental results have revealed, our method also leads to more accurate MAP predictions than Markov Logic Networks (MLNs) for small to moderately sized evidence sets. For larger evidence sets, MLNs lead to more accurate predictions, which is intuitively due to the fact that they are better equipped to aggregate large amounts of individually weak pieces of evidence. Inference in possibilistic logic is also considerably faster than methods for computing MAP queries from MLNs.

\section*{Acknowledgments}
This work was supported by a grant from the Leverhulme Trust (RPG-2014-164) and ERC Starting Grant 637277. JD is partially supported by the KU Leuven Research Fund  (C22/15/015), and FWO-Vlaanderen (G.0356.12, SBO-150033).

\ifappendix
\appendix



\section{Models Learned on UWCSE Dataset}

Table \ref{tab:uwcse} displays the possibilistic logic theory learned on the UWCSE dataset. The learned Markov logic network is displayed\footnote{We show the variables in the MLN as upper-case letters for consistency with the rest of the paper. However, note that it is more common for MLNs to use lower-case letters for variables.} in Table \ref{tab:uwcsemln}.

\begin{table*}[t]
\caption{The possibilistic logic theory learned in the UWCSE dataset (here $\lambda_{i} \leq \lambda_{i+1}$, hard rules of the theory are not shown and the variables in the rules are implicitly constrained by $\textit{AllDiff}$ constraints). }\label{tab:uwcse}
\begin{tabular}{ r } \hline
 \dots \quad\quad(116 hard constraints not shown here)\\

$(\textit{Publication}(V0, V1) \leftarrow \textit{TempAdvisedBy}(V3, V1) \wedge \textit{Publication}(V0, V3), \lambda_{16})$ \\
\quad\quad\quad\quad\quad\quad\quad\quad\quad\quad$(\textit{PostQuals}(V0) \leftarrow \textit{TempAdvisedBy}(V0, V3)\wedge \textit{TempAdvisedBy}(V4, V3) \wedge \textit{PostQuals}(V4), \lambda_{15})$ \\
$(\textit{AdvisedBy}(V0, V1) \leftarrow \textit{FacultyAffiliate}(V1), \textit{Publication}(V3, V1) \wedge \textit{Publication}(V3, V0), \lambda_{15})$ \\
$(\textit{PostGenerals}(V0) \leftarrow \textit{Publication}(V2, V0) \wedge \textit{AdvisedBy}(V0, V4), \lambda_{14})$ \\
$(\textit{FacultyAdjunct}(V0) \leftarrow \textit{AdvisedBy}(V2, V0), \textit{FacultyAdjunct}(V4) \wedge \textit{AdvisedBy}(V2, V4), \lambda_{13})$ \\
$(\textit{PostGenerals}(V0) \leftarrow \textit{TaughtBy}(V2, V0, V4)\wedge \textit{Student}(V0), \lambda_{13})$ \\
$(\textit{PostQuals}(V0) \leftarrow \textit{AdvisedBy}(V0, V3) \wedge \textit{ProjectMember}(V4, V3), \lambda_{12})$ \\
$(\textit{AdvisedBy}(V0, V1) \leftarrow \textit{Publication}(V3, V1) \wedge \textit{FacultyAdjunct}(V1) \wedge \textit{Publication}(V3, V0), \lambda_{11})$ \\
$(\textit{PostQuals}(V0) \leftarrow \textit{AdvisedBy}(V0, V3) \wedge \textit{PostQuals}(V4) \wedge \textit{AdvisedBy}(V4, V3), \lambda_{10})$ \\
$(\textit{Publication}(V0, V1) \leftarrow \textit{FacultyAffiliate}(V4), \textit{FacultyAffiliate}(V1) \wedge \textit{Publication}(V0, V4), \lambda_9)$ \\
$(\textit{PostGenerals}(V0) \leftarrow \textit{AdvisedBy}(V0, V3) \wedge \textit{PostGenerals}(V4) \wedge \textit{AdvisedBy}(V4, V3), \lambda_8)$ \\
$(\textit{Faculty}(V0) \leftarrow \textit{TempAdvisedBy}(V2, V0), \lambda_7)$ \\
$(\textit{PostQuals}(V0) \leftarrow \textit{AdvisedBy}(V0, V3) \wedge \textit{FacultyAdjunct}(V3), \lambda_6)$ \\
$(\textit{Faculty}(V0) \leftarrow \textit{AdvisedBy}(V2, V0), \lambda_5)$ \\
$(\textit{Publication}(V0, V1) \leftarrow \textit{FacultyAdjunct}(V4), \textit{FacultyAdjunct}(V1) \wedge \textit{Publication}(V0, V4), \lambda_4)$ \\
$(\textit{PreQuals}(V0) \leftarrow \textit{TempAdvisedBy}(V0, V3), \lambda_3)$ \\
$(\textit{FacultyAdjunct}(V0) \leftarrow \textit{Publication}(V2, V4), \textit{PreQuals}(V4) \wedge \textit{Publication}(V2, V0), \lambda_2)$ \\
 $(\textit{PostQuals}(V0) \leftarrow \textit{AdvisedBy}(V0, V3), \lambda_1)$ \\
 $(\bot, \lambda_\bot)$ \\
 \hline
\end{tabular}
\end{table*}

\begin{table*}[t]
\caption{The possibilistic logic theory learned in the UWCSE dataset (here $\lambda_{i} \leq \lambda_{i+1}$, hard rules of the theory are not shown and the variables in the rules are implicitly constrained by $\textit{AllDiff}$ constraints). }\label{tab:uwcsemln}
\begin{tabular}{ l r } \hline
0 &	$\textit{FacultyEmeritus}(V1)$ \\
-5.64478 &  $\textit{AdvisedBy}(V1,V2)$ \\
-5.10868 &  $\textit{Student}(V1)$ \\
-2.06003 &  $\textit{FacultyAffiliate}(V1)$ \\
-3.18786 &  $\textit{PostQuals}(V1)$ \\
-3.64133 &  $\textit{PostGenerals}(V1)$ \\
-0.80596 &  $\textit{FacultyAdjunct}(V1)$ \\
-0.0391718 & $\textit{TempAdvisedBy}(V1,V1)$ \\
-4.12274 &  $\textit{Publication}(V1,V2)$ \\
-0.0551981 & $\textit{AdvisedBy}(V1,V1)$ \\
-2.95067 &  $\textit{Faculty}(V1)$ \\
-6.56008 &  $\textit{ta}(V1,V2,V3)$ \\
-4.90068 &  $\textit{ProjectMember}(V1,V2)$ \\
-4.44769 &  $\textit{TempAdvisedBy}(V1,V2)$ \\
5.10868  &  $\textit{Professor}(V1)$ \\
-6.31879 &  $\textit{TaughtBy}(V1,V2,V3)$ \\
-4.69704 &  $\textit{PreQuals}(V1)$ \\
1.67999e-06 & $\neg \textit{TaughtBy}(V1,V2,V3) \vee \neg \textit{FacultyAdjunct}(V2) \vee \neg \textit{TempAdvisedBy}(V2,V4)$ \\ 
1.08275  &  $\neg \textit{TaughtBy}(V1,V2,V3) \vee \neg \textit{FacultyAdjunct}(V2) \vee \neg \textit{Publication}(V4,V2)$ \\
6.07986  &  $\textit{Faculty}(V1) \vee \neg \textit{TempAdvisedBy}(V2,V1) \vee \textit{FacultyAffiliate}(V1)$ \\
2.48337  &  $\textit{PreQuals}(V1) \vee \textit{PostQuals}(V1) \vee \textit{ProjectMember}(V2,V1) \vee \neg \textit{TempAdvisedBy}(V1,V3)$ \\
9.33506  &  $\textit{PreQuals}(V1) \vee \textit{PostQuals}(V1) \vee \neg \textit{AdvisedBy}(V1,V2) \vee \textit{PostGenerals}(V1)$ \\
-12.0071 &  $\neg \textit{Student}(V1) \vee \textit{FacultyAdjunct}(V1) \vee \textit{Professor}(V1) \vee \textit{Faculty}(V1) \vee \textit{FacultyAffiliate}(V1)$ \\
-2.05975 &  $\textit{PostQuals}(V1) \vee \textit{FacultyAdjunct}(V1) \vee \neg \textit{Faculty}(V1) \vee \neg \textit{AdvisedBy}(V2,V1) \vee \textit{FacultyAffiliate}(V1)$ \\
 \hline
\end{tabular}
\end{table*}

\section{Algorithms for Parameter Estimation}\label{sec:algorithms}

In this section we provide additional details about algorithms for estimating the parameters of the maximum likelihood geometric programming problem stated in Section \ref{secWeightLearning}.

\subsection{Subset Counting}

Let $\alpha$ be an existentially quantified conjunction (i.e.\ a {\em conjunctive query}) and let us write $\textit{vars}(\alpha)$ for the set of variables occurring in $\alpha$. 
Then the \emph{$k$-extension} of $\alpha$ is the conjunction $\exists V_1',V_2',\dots,V_k' : \alpha \wedge \textit{card}(k,\allowbreak V_1,\allowbreak \dots,\allowbreak V_m,\allowbreak V_1',\allowbreak \dots,\allowbreak V_k')$ where $V_1,\dots,V_m$ are the variables contained in $\textit{vars}(\alpha)$, $V_1',\dots,V_k'$ are new variables and the constraint $\textit{card}(k,\allowbreak V_1,\allowbreak \dots,\allowbreak V_m,\allowbreak V_1',\allowbreak \dots,\allowbreak V_k')$ is true iff the set $\{V_1,\allowbreak \dots,\allowbreak V_m,V_1',\allowbreak \dots,\allowbreak V_k'\}$ has cardinality $k$.

\begin{definition}[Matching subset counting]
Consider an example $\Upsilon = (\mathcal{A},\mathcal{C})$, $k\in\mathbb{N}$, and an existentially quantified conjunction $\alpha$. The matching subset counting problem is to count the cardinality of the set $\{ \mathcal{S} | \mathcal{S} \subseteq \mathcal{C}, |\mathcal{S}| = k, \Upsilon\langle \mathcal{S} \rangle \models \alpha \}$. The elements $\mathcal{S}$ of this set are called matching subsets for $\alpha$.
\end{definition}

A naive solution would be to use a CSP solver \cite{csp} to obtain all tuples that make the $k$-extension of $\alpha$ satisfied, and to convert these into sets. In general, this approach actually has an optimal worst-case complexity when the CSP solver cannot do better than exhaustive search. However, its disadvantage is that even if the $k$-extension of $\alpha$ contained just $k$ variables, this method might discover each matching subset $k!$ times. Moreover, for conjunctive queries whose corresponding CSP problem can be solved faster than by brute-force search (e.g.\ for bounded hypertree-width formulas \cite{gottlob2014treewidth}) this approach may not have worst-case optimality. Therefore, we now describe another approach which also relies on CSP solving, but which needs to generate fewer solutions per matching set in the worst case. 

The pseudocode of our proposed algorithm for enumerating matching subsets is shown in Algorithm \ref{alg:subsets}. The algorithm assumes access to a CSP solver which is capable of answering the following type of queries $\mathbf{CSP}(\alpha,V,\Upsilon)$: given a global example $\Upsilon = (\mathcal{A},\mathcal{C})$, a conjunctive query $\alpha$ and a variable $V \in \textit{vars}(\alpha)$, return the set of all constants $c \in \mathcal{C}$ such that $\mathcal{A} \models \alpha\theta$ where $\theta = \{ V/c \}$ is a substitution. Such queries can either be implemented directly in the solver or by making multiple calls to an external CSP solver; in our implementation we have adopted the former approach. The algorithm uses these queries to iteratively build the set of matching subsets. It can be shown that it generates at most $2^k$ different solutions of the CSP problems per matching subset, which even for small values of $k$ is considerably better than the $k!$ solutions generated per matching subset by the naive method (e.g.\ $2^6 = 64$ whereas $6! = 720$). For CSP problems from a tractable subclass, Algorithm \ref{alg:subsets}'s runtime can be shown to be bounded by $O(2^k \cdot |T| \cdot f_{CSP}(k,n) + f_{pre}(k,n))$ where $T$ is the output, $n$ is the length of the input, $f_{CSP}(k,n)$ is the time needed for checking the existence of a solution of the CSP problem (after preprocessing) and $f_{pre}(k,n)$ is the time for preprocessing (e.g.\ for establishing global consistency).

\begin{algorithm}[tb]
\caption{\sc All-Matching-Subsets}
\label{alg:subsets}
{\small
\begin{algorithmic}[1]
\Require{Global example $\Upsilon$, conjunctive query $\alpha$ (which is assumed to be a $k$-extension), integer $k > 0$.}
\Ensure{The set of all matching $k$-subsets of $\Upsilon$ for the query $\alpha$.}
\State{{\bf Let} $V_1, \dots, V_m$ be a list of all variables from $\alpha$, $\textit{Current} := \{ \emptyset \}$}
\State{{\bf For} $i = 1, \dots, m$ {\bf do}}
\State{\quad$\textit{Next} := \emptyset$}
\State{\quad{\bf Foreach} $\mathcal{S} \in \textit{Current}$ {\bf do}}
\State{\quad\quad {\bf Let} $\alpha' := \alpha \wedge \textit{in}(V_1,\mathcal{S}) \wedge \dots \wedge \textit{in}(V_{i-1},\mathcal{S})$ /* constraints $\textit{in}(V,\mathcal{S})$ require the value of $V$ to be in $\mathcal{S}$ */}
\State{\quad\quad {\bf Let} $\mathcal{P} := \mathbf{CSP}(\alpha',V_i,\Upsilon)$.}
\State{\quad\quad $\textit{Next} := \textit{Next} \cup \{ \mathcal{S}\cup \{c\} | c \in \mathcal{P} \}$.}
\State{\quad{\bf EndForeach}}
\State{\quad{$\textit{Current} := \textit{Next}$}}
\State{{\bf EndFor}}
\State{{\bf Return} $\textit{Current}$}
\end{algorithmic}}
\end{algorithm}

For larger datasets, counting the number of matching subsets exactly is not feasible. Therefore we next describe an approximate counter, which relies on a recent algorithm called ApproxMC2 \cite{chakraborty2016algorithmic}. ApproxMC2 is a method for propositional model counting. Since it uses SAT solvers as black-box oracles, it can be used for other problems as long as their solutions can be represented as fixed-length boolean vectors. In particular, it can be naturally used for approximate subset counting, since we can represent the $k$-subsets as Boolean vectors in which exactly $k$ entries are true. These entries correspond to indicator variables which are true if the corresponding constants are contained in the subset. To apply ApproxMC2, which relies on partitioning the space by long XOR constraints, all we need is to support the propagation of XOR constraints in the CSP solver. We propagate global XOR constraints using Gaussian elimination over GF[2]. We do not represent the indicator variables explicitly as decision variables of the CSP problem, but only in the equations over GF[2] inside the global constraint. If a constant is not present in domain of any variable then the corresponding indicator variable must be set to false. Likewise if a variable is assigned a constant $c$ then the corresponding indicator variable must be set to true. If an indicator variable is forced to be zero by Gaussian elimination, we can remove it from all domains.

In many cases, when the density of matching subsets is high, we can get reliable estimates in a much simpler way. In particular, we also implemented a third method which relies on uniformly drawing a small number of $k$-subsets from $\Upsilon$ and estimating the confidence intervals. If the relative size of the confidence interval is below a pre-fixed threshold, we use this value instead of employing the other, more costly methods.

Our overall approach is then to first try the exact method with a small maximum limit. If this fails, we use the simple sampling-based algorithm. If that fails as well, we try the exact method with larger maximum limit. Finally, if this fails as well, we use the algorithm based on ApproxMC2.

\subsection{Model Counting}

To compute the parameters of the optimization problem for likelihood maximization, we also need to compute model counts for the theories consisting of non-ground clauses over sets of $k$ constants. To this end, we implemented two algorithms based on the ApproxMC2 algorithm. The first version, which works best for small $k$, simply grounds the theory and computes the model count using ApproxMC2. The second version exploits the fact that the ApproxMC2 algorithm runs a SAT solver to decide if the theory has at most $t$ solutions (where $t$ is a threshold based on given tolerance parameters). It is not always necessary to ground the whole theory to check this. Therefore this second version uses a strategy based on cutting plane inference \cite{riedel08}, grounding the rules iteratively only when they become violated, where checking whether a non-ground rule is violated can be done using a CSP solver.




\fi

\bibliographystyle{named}
\bibliography{bibliography}

\end{document}